\documentclass{article}

\usepackage{microtype}
\usepackage{graphicx}
\usepackage{subfigure}
\usepackage{booktabs} 

\usepackage{hyperref}

\usepackage[accepted]{icml2025}

\usepackage{amsmath}
\usepackage{amssymb}
\usepackage{mathtools}
\usepackage{amsthm}

\usepackage{subcaption}
\usepackage{multirow}

\usepackage[capitalize,noabbrev]{cleveref}

\theoremstyle{plain}
\newtheorem{theorem}{Theorem}[section]
\newtheorem{proposition}[theorem]{Proposition}

\theoremstyle{definition}

\theoremstyle{remark}

\usepackage[textsize=tiny]{todonotes}

\icmltitlerunning{WaveGAS: Waveform Relaxation for Scaling Graph Neural Networks}

\begin{document}

\twocolumn[
\icmltitle{WaveGAS: Waveform Relaxation for Scaling Graph Neural Networks}




\begin{icmlauthorlist}
\icmlauthor{Jana Vatter}{tum}
\icmlauthor{Mykhaylo Zayats}{ibm}
\icmlauthor{Marcos Martínez Galindo}{ibm}
\icmlauthor{Vanessa López}{ibm}
\icmlauthor{Ruben Mayer}{uba}
\icmlauthor{Hans-Arno Jacobsen}{uto}
\icmlauthor{Hoang Thanh Lam}{ibm}
\end{icmlauthorlist}

\icmlaffiliation{tum}{Technical University of Munich, Germany}
\icmlaffiliation{ibm}{IBM Research, Dublin, Ireland}
\icmlaffiliation{uto}{University of Toronto, Canada}
\icmlaffiliation{uba}{University of Bayreuth, Germany}

\icmlcorrespondingauthor{Jana Vatter}{jana.vatter@tum.de}

\icmlkeywords{Graph Neural Networks, Waveform Relaxation}

\vskip 0.3in
]



\printAffiliationsAndNotice{}  

\begin{abstract}
With the ever-growing size of real-world graphs, numerous techniques to overcome resource limitations when training Graph Neural Networks (GNNs) have been developed. One such approach, GNNAutoScale (GAS), uses graph partitioning to enable training under constrained GPU memory. GAS also stores historical embedding vectors, which are retrieved from one-hop neighbors in other partitions, ensuring critical information is captured across partition boundaries. The historical embeddings which come from the previous training iteration are stale compared to the GAS estimated embeddings, resulting in approximation errors of the training algorithm. Furthermore, these errors accumulate over multiple layers, leading to suboptimal node embeddings. To address this shortcoming, we propose two enhancements: first, WaveGAS, inspired by waveform relaxation, performs multiple forward passes within GAS before the backward pass, refining the approximation of historical embeddings and gradients to improve accuracy; second, a gradient-tracking method that stores and utilizes more accurate historical gradients during training. Empirical results show that WaveGAS enhances GAS and achieves better accuracy, even outperforming methods that train on full graphs, thanks to its robust estimation of node embeddings.
\end{abstract}

\section{Introduction}
GNNs have gained widespread popularity in numerous domains \cite{ying2018graph, reiser2022graph, buterez2024transfer, zhou2024reconstructed}. Many real-world graphs consist of billions of nodes and edges \cite{hu2021ogb, snapnets, gupta2013wtf}. Training GNNs on these large-scale graphs presents significant challenges due to resource limitations. State-of-the-art techniques involve partitioning these graphs into smaller subgraphs or mini-batches that fit within GPU memory constraints and training on each individual mini-batch separately \cite{fey2021gnnautoscale, chiang2019cluster, zeng2020graphsaint, zheng2022distributed, zheng2022bytegnn, yang2023betty, li2021training}. However, this sampling-based approach incurs information loss, as it tends to overlook the propagation of information across different partitions. To mitigate this challenge, the GNNAutoScale (GAS) \cite{fey2021gnnautoscale} approach employs a historical embedding technique to maintain a record of node embeddings in a secondary memory. During training on a particular mini-batch, historical embeddings of one-hop nodes from other mini-batches that are directly connected to nodes within the target mini-batch are retrieved and transferred to the GPUs. By approximating embeddings of nodes that are outside of the mini-batch, GAS rivals the results obtained using the entire graph training methods. 

GAS has proven to be effective on various benchmark datasets, but despite its innovative design, it also has certain limitations. The historical embeddings which come from the previous training iteration are stale compared to the embedding estimated within current iteration. As discussed in detail in Section \ref{sec:gas_flaws}, the staleness corrupts the computation of embeddings in the forward pass of the model and also corrupts gradients of the loss function used to update model weights in the backward pass. The resulting errors in the GAS training procedure accumulate over multiple layers and lead to deviations in the obtained node embeddings and suboptimal results compared to full-batch training.

Our work explores strategies to overcome inherent errors of the GAS approach and, in particular, to mitigate staleness of historical embeddings. To this end, we explore parallels between GNN's forward pass and forward solvers of systems of Ordinary Differential Equations (ODE) \cite{xhonneux2020continuous, poli2020graph}. We get inspiration from the Waveform Relaxation (WR) \cite{white1987waveform}, a technique proven to efficiently solve large systems of ODEs by iteratively solving subsystems and updating solutions. Borrowing from the WR methodology, we introduce WaveGAS, an approach which extends GAS by additionally running multiple forward passes to achieve a more frequent update of historical embedding before running a single backward pass. WaveGAS reduces the staleness of the historical embeddings and subsequently allows us to obtain more accurate embeddings. In section \ref{sec:methods} we introduce WaveGAS and provide intuition for its staleness mitigation mechanisms. In section \ref{sec:results}, we demonstrate its efficiency by running a set of experiments, which indeed confirm that WaveGAS achieves higher accuracy than GAS while maintaining the same memory footprint, albeit a trade-off with a proportionally longer runtime.

\section{Background}
\label{sec:backgrounds}

\begin{figure*}[htb]
\center{\includegraphics[width=1.0\textwidth]{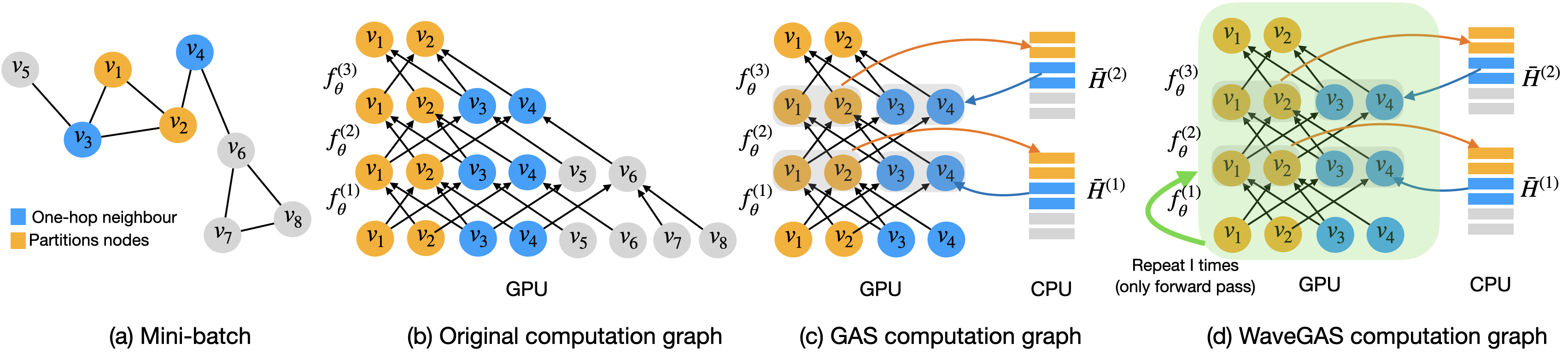}}
\caption{\label{fig:gas} This figure depicts: (a) A graph with mini-batch selection, (b)its corresponding 2-layer GNN full computational graph, (c) a partial computational graph utilizing historical embeddings on the CPU, and (d) our enhanced WaveGAS algorithm , inspired by WR. WaveGAS performs multiple forward passes through the partial computational graphs to refine historical embeddings before executing a backward pass to update the network parameters.}
\end{figure*}

In this section, we provide a generic formulation of the GNN and GAS \cite{fey2021gnnautoscale} approaches and then describe the limitations of GAS that motivated our work.

\subsection{GNNs formulation}
GNNs are a type of neural network designed to handle tasks on graph structured data. Let $\mathcal{G}=(\mathcal{V}, \mathcal{E})$ denote a graph where $\mathcal{V}$ is a set of nodes and $\mathcal{E}$ is a set of edges. GNNs compute the embedding $h_v$ of each node $v\in\mathcal{V}$ based on its connections to other nodes. The parameters of GNNs are optimized by minimizing a task-specific loss function, typically for node-level classification or prediction tasks, though other objectives are also possible.

In this work, we consider a broad class of GNN models that follow a message-passing scheme \cite{gilmer2017neural}. In this framework, the forward pass of the model for computing the embeddings of a node $v$ is defined as:
\begin{equation}
\label{eq:forward_pass}
\begin{aligned}
    h^{t,l+1}_{v} &= f_{\theta_t}^{l+1}\left(h^{t,l}_v, m_v^{t,l} \right) \\
    m_v^{t,l} &= \bigoplus \Bigl( \Bigl\{ g_{\theta_t}^{l+1}(h^{t,l}_{v}, h^{t,l}_{w})\Bigr\}_{w \in N(v)} \Bigr)
\end{aligned}
\end{equation}
where $l+1$ denotes network layer index and $t$ is an optimization iteration index (epoch). $\theta$ represents the GNN model weights parameterizing differentiable message-passing functions $f_{\theta_t}^{l+1}(\cdot)$ and $g_{\theta_t}^{l+1}(\cdot)$. $\bigoplus$ is a permutation-invariant aggregation function typically taken as the sum, mean or maximum and operates on multisets of messages incoming from a set $N(v)$ of neighbors of node $v$. Different choices of $f_{\theta}(\cdot)$, $g_{\theta}(\cdot)$ and $\bigoplus$ result in different GNN architectures such as Graph Convolutional Network (GCN)~\cite{kipf2016semi} or Graph Attention Network (GAT)~\cite{velivckovic2018graph}. For an overview of possible GNN architectures the reader is referred to \citet{wu2020comprehensive,zhou2020graph}.

For the backward pass step, one needs to compute gradients of a given loss function $\mathcal{L}(\cdot) = \mathcal{L}(\{h_v\}_{v\in\mathcal{V}})$ with respect to the model parameters $\theta$. It is important to note that the embedding vectors $h_v^{t, l}$ are composite vector-functions of a variable $\theta$. Thus, following the chain rule this would require computation of the following terms:
\begin{equation}\begin{aligned}
\label{eq:grad_terms}
    \frac{d g_{\theta_t}^{l+1}(h^{t,l}_{v}, h^{t,l}_{w})}{d\theta}&=\left.\frac{d g_{\theta_t}^{l+1}(h, h^{t,l}_{w})}{d h}\right\rvert_{h=h^{t,l}_{v}} \frac{d h^{t,l}_{v}}{d\theta} \\
    &+ \left.\frac{d g_{\theta_t}^{l+1}(h^{t,l}_{v}, h)}{d h}\right\rvert_{h=h^{t,l}_{w}} \frac{d h^{t,l}_{w}}{d\theta}
\end{aligned}\end{equation}
for some given $\theta$ and for each $v\in\mathcal{V}$ and $w\in N(v)$.

\subsection{GAS formulation}
For mini-batch computation of node embeddings, let $\mathcal{B}\subset \mathcal{V}$ denote a subset of graph nodes or a mini-batch. The neighborhood of each node $v$ can then be divided into two parts: the first containing nodes within the mini-batch, i.e., $N(v)\cap B$, and the second containing nodes outside the mini-batch, i.e., $N(v)\backslash B$. Using this partition, the aggregated message computation can be rewritten as:
\begin{align}
    m_v^{t,l} = \bigoplus \Bigl(
    & \Bigl\{ g_{\theta_t}^{l+1}(h^{t,l}_{v}, h^{t,l}_{w})\Bigr\}_{w \in N(v)\cap B} \\ 
    \cup & \Bigl\{ g_{\theta_t}^{l+1}(h^{t,l}_{v}, h^{t,l}_{w}) \Bigr\}_{w \in N(v)\backslash B} \Bigr) \label{eq:ext_nodes_messeges}
\end{align}

A naive mini-batching approach drops the term in equation~\eqref{eq:ext_nodes_messeges} corresponding to nodes outside $\mathcal{B}$ during the forward pass. While this enables independent computation of node embeddings in mini-batches, it also removes a potentially significant amount of information from the model.

In contrast, GAS approximates the embeddings of one-hop neighbors outside the mini-batch using historical embeddings, which are taken as embeddings computed during the previous optimization iteration and stored in secondary memory, mimicking training on the entire graph (Figure \ref{fig:gas}.b). During training on a specific mini-batch $\mathcal{B}$ (as illustrated in Figure \ref{fig:gas}.c) historical embeddings of its one-hop neighbors are loaded into GPU memory and GNN computes node embeddings for nodes within $\mathcal{B}$ which are then stored back in the historical embeddings.

Let $\bar{h}_v^{t,l} = h_v^{t-1,l}$ denote the historical embedding of node $v\in\mathcal{B}$ at layer $l$ after $t-1$ epochs. The aggregated message $m_v^{t,l}$ can then be approximated as:
\begin{equation}\begin{aligned}
\label{eq:gas_mpass}
    m_v^{t,l} \approx \bigoplus \Bigl(
    & \Bigl\{ g_{\theta_t}^{l+1}(h^{t,l}_{v}, h^{t,l}_{w})\Bigr\}_{w \in N(v)\cap B} \\ 
    \cup & \Bigl\{ g_{\theta_t}^{l+1}(h^{t,l}_{v}, \bar h^{t,l}_{w}) \Bigr\}_{w \in N(v)\backslash B} \Bigr)
\end{aligned}\end{equation}

This approach preserves much of the information from nodes outside the mini-batch while maintaining computational efficiency. 

\subsection{GAS flaws}
\label{sec:gas_flaws}

The main flaw of the GAS approach is that historical embeddings $\bar{h}_v^{t,l}$ run stale compared to the GAS embedding $\tilde{h}_v^{t,l}$ approximating full graph embeddings $h_v^{t,l}$. That is $\bar{h}_v^{t,l} \neq \tilde h_v^{t,l}$ and $0 \leq \|\bar{h}_v^{t,l} - \tilde h_v^{t,l}\| \leq \epsilon$, since historical embeddings are not updated within the current optimization step. This staleness introduces disturbances into the forward and backward passes of the training algorithm.

In the forward pass, they disturb the second term in the expression \eqref{eq:gas_mpass} for computing the aggregated message $m_v^{t,l}$. This induces approximation errors $\|\tilde{h}_v^{t,l} - h_v^{t,l}\|$ and, as demonstrated by~\cite{fey2021gnnautoscale}, these errors accumulate across multiple layers and depend on the level of staleness $\epsilon$ and the Lipschitz constants of the learned message-passing functions.

In the backward pass, the staleness disturbs the computation of the components of the loss function gradient~\eqref{eq:grad_terms}. Indeed, GAS approximates the first term in~\eqref{eq:grad_terms} by means of historical embeddings as follows:
\begin{equation}
    \left.\frac{d g_{\theta_t}^{l+1}(h, h^{t,l}_{w})}{d h}\right\rvert_{h=h^{t,l}_{v}} \frac{d h^{t,l}_{v}}{d\theta}
    \approx \left.\frac{d g_{\theta_t}^{l+1}(h, \bar h^{t,l}_{w})}{d h}\right\rvert_{h=h^{t,l}_{v}} \frac{d h^{t,l}_{v}}{d\theta}
\end{equation}
Furthermore, because historical embeddigs are computed at the previous optmization epoch they are treated as a constant with respect to current epoch parameters $\theta$ thereby $\frac{d \bar h^{t,l}_{w}}{d\theta}=0$ and
\begin{equation}
    \left.\frac{d g_{\theta_t}^{l+1}(h^{t,l}_{v}, h)}{d h}\right\rvert_{h=\bar h^{t,l}_{w}} \frac{d \bar h^{t,l}_{w}}{d\theta} = 0
\end{equation}
leading to the omission of the second term in~\eqref{eq:grad_terms} when computed within GAS approach.

This demonstrates that the gradient of the loss function with respect to the parameters $\theta$ in the mini-batch computation graph is only an approximation of the gradient in the full-batch computation graph. The level of gradient approximation depends on the number of nodes at the interface between mini-batches, i.e., the total size of multisets $N(v)\backslash B$, and subsequently on the staleness level of the corresponding historical enbeddings and smoothness of the corresponding GAS estimated embeddings.

In summary, the staleness of the historical embeddings plays an important role in inducing approximation errors of the node embeddings by corrupting both forward and backward passes of the training algorithm. The authors of the original work~\cite{fey2021gnnautoscale} suggested two ways to tighten those approximation errors: i) the first one uses a sophisticated graph partitioning strategy to minimize the number of nodes at the interface between mini-batches; ii) the second one adds an additional regularization term to the loss function enforcing local Lipschitz continuity of the message-passing functions. They, however, provided no recipe for mitigating staleness of historical embeddings. Reducing the staleness of historical embeddings, and thereby minimizing the approximation error, is the primary motivation for our work.

\section{Methods}
\label{sec:methods}
To address the issue of staleness in historical embeddings and, consequently, improve the quality of the computed embeddings, we propose the WaveGAS approach.

\subsection{GNN as an ODE discretization}
The structural similarity between the forward pass of certain types of GNNs and ODEs has been previously noted, for example, in~
\cite{xhonneux2020continuous, poli2020graph} with the goal of designing novel GNN architectures that resemble ODE. In this work, we take a similar perspective: we borrow ideas from numerical methods for solving ODE to develop a more computationally efficient approach for GNNs. To this end, we first establish connections between a class of GNNs and discretized ODE.

\begin{proposition} \label{prop:gnn_as_ode} If the update function $f_{\theta}^{l}(\cdot,\cdot)$ is defined as the sum of its arguments, then the forward pass of the GNN in~\eqref{eq:forward_pass} can be interpreted as a discretization of an ODE system.
\end{proposition}
\begin{proof}
During the forward pass, the GNN parameters $\theta$ and the epoch index $t$ are fixed. For simplicity, we omit them and use the assumption about $f^{l}(\cdot,\cdot)$ to write:
\begin{equation}
\begin{aligned}
    h^{l+1}_{v} &= h^{l}_v + m^{l}(h^{l}_{v}, \{h^{l}_{w}\}_{w \in N(v)}) \\
    m^{l}(h_{v}, H_v) &= \bigoplus \Bigl( \Bigl\{ g^{l+1}(h_{v}, h_{w})\Bigr\}_{h_w \in H_v} \Bigr)
\end{aligned}
\end{equation}

Assuming $\Delta\tau=1$, the layer index $l$ can be treated as a discrete index of a continuous variable $\tau$, analogous to time in ODE and $h^{l+1}_{v}$ can be seen as a discretization of a continuous vector function $h_v(\tau)$. Rolling this discretization back to the continuous domain, we obtain the ODE:
\begin{equation}
\begin{aligned}
    \frac{dh_v(\tau)}{d\tau} &= m(\tau, h_{v}(\tau), \{h_{w}(\tau)\}_{w \in N(v)}) \\
    m(\tau, h_{v}(\tau), H_v(\tau)) &= \bigoplus \Bigl( \Bigl\{ g(\tau, h_{v}(\tau), h_{w}(\tau))\Bigr\}_{h_w \in H_v} \Bigr)
\end{aligned}
\end{equation}
where $\tau\in[0,L]$ and $g(\cdot,\cdot,\cdot)$ is a non-unique continuous over $\tau$ function which corresponds to the discrete function $g^{l+1}$ such that
\begin{equation}
    g(\tau_{l},h_v(\tau_l),h_w(\tau_l)) = g^{l+1}(h_{v}^l, h_{w}^l)    
\end{equation}
\end{proof}

ODE systems can be computationally demanding due to their large size. However, they can be efficiently solved in batches using the WR method. This iterative method alternates between fixing one variable (aka loading it from history) and solving for the others providing guarantees of solution convergence. Although Proposition~\ref{prop:gnn_as_ode} applies only to a specific class of GNNs, it offers a strong intuition for how WR can be adapted to run the forward pass computation of GNNs. This idea forms the basis for our WaveGAS approach.

\subsection{WaveGAS formulation}
WaveGAS extends the GAS approach by incorporating ideas from the WR method used to solve large systems of ODEs. Rather than performing a single forward pass during each training step, WaveGAS performs multiple forward passes to achieve more frequent updates to historical embeddings before running a single backward pass.

Let $s=[0,\cdots,I-1]$ represent the index of the waveform iteration and let $\bar h^{t,l,s}_{v}$ denote the historical embedding in WaveGAS. For $s=0$ we initialize $\bar h^{t,l,0}_{v}=\bar h^{t,l}_{v}$ where $\bar h^{t,l}_{v}$ is the historical embedding used in GAS. For each waveform iteration $s$ we iterate over mini-batches, computing their nodes embeddings and upgrading the corresponding historical embeddings for each layer as follows:
\begin{equation}
    h^{t,l+1}_{v} = f_{\theta_t}^{l+1}\left(h^{t,l}_v, m_v^{t,l} \right)
\end{equation}
\begin{equation}
\begin{aligned}
    m_v^{t,l} \approx \bigoplus \Bigl(
    & \Bigl\{ g_{\theta_t}^{l+1}(h^{t,l}_{v}, h^{t,l}_{w})\Bigr\}_{w \in N(v)\cap B} \\
    \cup & \Bigl\{ g_{\theta_t}^{l+1}(h^{t,l}_{v}, \bar h^{t,l,s}_{w}) \Bigr\}_{w \in N(v)\backslash B} \Bigr)
\end{aligned}
\end{equation}
\begin{equation}
    \bar h^{t,l+1,s+1}_{v} = h^{t,l+1}_{v}
\end{equation}

This process runs without tracking gradients and is designed solely to improve the historical embeddings. Once the waveform iterations are complete, a final forward pass is conducted with tracking gradients, followed by a backward pass to update model parameters $\theta$. The algorithm of WaveGAS is outlined in Algorithm~\ref{alg:wavegas} and visualized in Figure~\ref{fig:wavegasvis}.

The number of waveform iterations $I$ plays an important role in WaveGAS, as it prescribes how many times historical embeddings are updated within one round of parameters update. Larger $I$ allows for more updates between mini-batches but increases computational cost. Importantly, the memory footprint of WaveGAS remains the same as that of GAS, regardless of the value of $I$. Finally, when $I=1$, i.e., 1 waveform iteration is used, WaveGAS reduces to the standard GAS algorithm.
\begin{figure}[htb]
\center{\includegraphics[width=1.0\columnwidth]{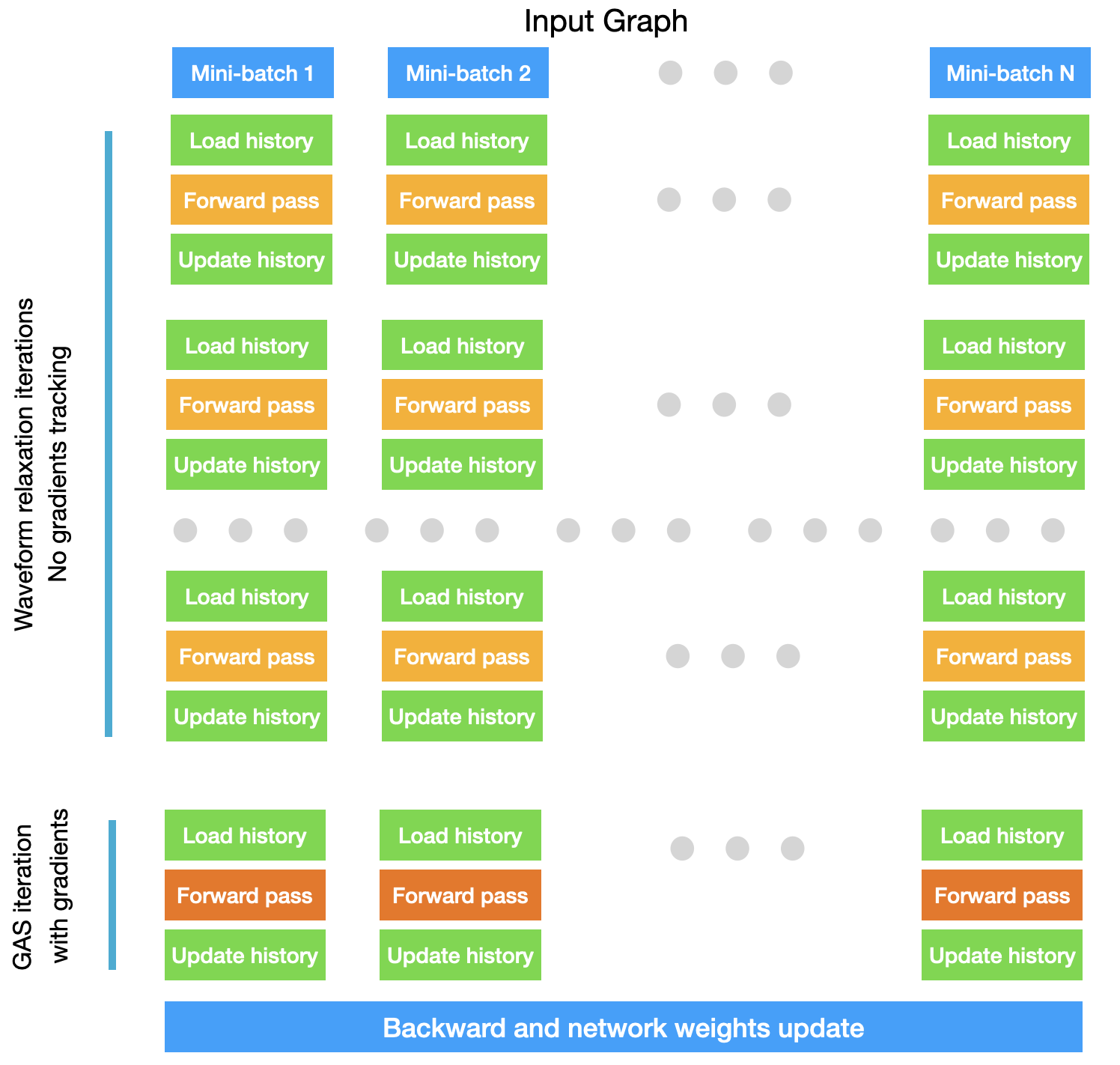}}
\caption{\label{fig:wavegasvis} A visualization of the WaveGAS algorithm running waveform iterations.}
\end{figure}

\begin{algorithm}[H]
    \caption{Training Loop for our WaveGAS method}
    \label{alg:wavegas}
\begin{algorithmic}[1]
    \STATE \texttt{// Training data $\mathcal{D}_{train}$, GNN model $model$, loss function $\mathcal{L}$, epochs $N$, WaveGAS iterations $I$, layers $L$}
    \FOR{epoch $= 1$ to $N$}
    \STATE \texttt{// WaveGAS Iterations}
        \FOR{WaveGAS iteration $= 1$ to $I$}
            \FOR{each batch in $\mathcal{D}_{train}$}
            \STATE \texttt{// Forward Pass}
            \STATE with $\text{torch.no\_grad()}$: $\hat{Y} \gets model(\text{batch})$
                \FOR{layer in $L$}
                \STATE Push and pull historical embeddings\;
                \ENDFOR
            \ENDFOR
        \ENDFOR

    \STATE \texttt{// Main Training Loop}
    \FOR{each $(\text{batch}, Y)$ in $\mathcal{D}_{train}$}
        \STATE \texttt{// Forward Pass}
        \STATE $\hat{Y} \gets model(\text{batch})$\;
        \FOR{layer in $L$}
            \STATE Push and pull historical embeddings\;
        \ENDFOR
        \STATE Compute loss: $\mathcal{L}_{\text{batch}} \gets \mathcal{L}(\hat{Y}, Y)$\;
        \STATE Zero gradients: $\text{optimizer.zero\_grad}()$\;
        \STATE Backward pass: $\mathcal{L}_{\text{batch}}.\text{backward}()$\;
        \STATE Update parameters: $\text{optimizer.step}()$\;
    \ENDFOR
        
    \ENDFOR
\end{algorithmic}
\end{algorithm}

\subsection{WaveGAS convergence}
The main advantage of the WaveGAS approach is that both historical embeddings and computed node embeddings undergo multiple rounds of updates that progressively align them closer to each other. This process reduces the staleness in historical embeddings, thereby improving the quality of the resulting embeddings.

To quantify the rate at which staleness declines, we again refer to WR method, which has theoretical convergence bounds. In WR, staleness is analogous to the solution convergence rate. Applying this bound to GNN satisfying conditions of Proposition~\ref{prop:gnn_as_ode} provides the following result:
\begin{equation}
\label{eq:convergence_rate}
    \|\bar h^{t,L,I}_{v} - \hat h^{t,L}_{v}\| \leq \frac{(CL)^I}{I!} \|\bar h^{t,L,0}_{v} - \hat h^{t,L}_{v}\|
\end{equation}
Here $\hat h^{t,L}_{v}$ denotes the embeddings obtained by WaveGAS, $L$ is the number of GNN layers and $C$ is a constant related to the Lipschitz continuity of the message-passing functions. Since the term $\|\bar h^{t,L,0}_{v} - \hat h^{t,L}_{v}\|$ represents the staleness level in GAS approach, the equation~\eqref{eq:convergence_rate} demonstates that the staleness in WaveGAS decreases superlinearly and is significantly lower then in GAS. It further shows that even relatively small number of waveform iterations $I$ can make staleness levels negligible. It is important to note, however, that only certain GNNs satisfy the conditions of Proposition~\ref{prop:gnn_as_ode}. While this does not limit the applicability of WaveGAS to more general GNNs arcitectures, determining staleness convergence rate for such cases remains an open problem for future work.

Additionally, an analysis of WaveGAS convergence to the full graph trained embeddings must account for the loss function gradient computation. While WaveGAS reduces the approximation error in the first term of~\eqref{eq:grad_terms} due to reduced staleness, it does not address the omission of the second term. As a result, while WaveGAS improves upon GAS, it still relies on approximations.

\subsection{GradAS formulation}
One potential method to address the omission of the second term in the loss function gradient formulation~\eqref{eq:grad_terms} is to cache the partial derivatives \(\frac{d h^{t,l}_{v}}{d\theta}\) in the history, similarly to how historical embeddings are handled. This approach, referred to as Gradient Auto Scaling (GradAS), aims to provide a more accurate gradients computation.

The main bottleneck of GradAS lies in the required cache size, which promptly becomes impractically large, scaling with the product of the number of nodes and the total number of model parameters. This makes loading historical gradients into GPU memory impossible for all but very small mini-batches. Consequently, the input graph must be partitioned into a very large number of mini-batches. In its turn, this inflates errors due to a high number of the nodes on the interface between mini-batches and may diminish benefits of correcting gradient terms. The substantial computational overheads make GradAS highly sensitive to the graph partitioning settings limiting its usability.

\section{Experiments}
In this work, we focus on improving the GAS method by introducing our WaveGAS approach. To evaluate its performance, we compare WaveGAS against the original GAS model and full-graph training. Our experiments leverage multiple datasets from diverse domains, including citation graphs \cite{sen2008collective, yang2016revisiting} and co-purchasing networks \cite{shchur2018pitfalls}. These datasets are standard benchmarks, and we adopt the same training, validation, and test splits as proposed in \cite{fey2021gnnautoscale}. Dataset characteristics and hyperparameter configurations are detailed in Table \ref{tab:datasets}. For GAS, we follow the open-source implementation and guidelines provided by the original codebase\footnote{\url{https://github.com/rusty1s/pyg_autoscale}}.

The hyperparameter choices are based on the recommended settings for the GAS model \cite{fey2021gnnautoscale}. A batch size of $n$ indicates that $n$ training samples are processed together before proceeding to the next batch. The number of partitions specifies how many subgraphs the full graph is divided into, while the number of iterations refers to the forward passes performed when using the waveform relaxation method. For WaveGAS, we varied the number of iterations from 1 to 10, selecting the optimal number based on the validation set performance, with the final results reported on the test set.

We employ a two-layer GCN \cite{kipf2016semi} as the GNN architecture. Each model is trained for 200 epochs, and all experiments are repeated 20 times to ensure reliability. The reported results include the average test accuracy with standard deviations. The optimal number of WaveGAS iterations is determined using validation accuracy. All training is conducted on a single Nvidia RTX A6000 (48GB), with training histories stored in RAM.

\begin{table*}[t]
    \centering
    \caption{Dataset characteristics and default hyperparameter configuration followed \cite{fey2021gnnautoscale}}
    \begin{tabular}{|c|c|c|c|c|c|c|c|c|}
        \hline
        \textbf{Dataset} & \textbf{Nodes} & \textbf{Edges} & \textbf{Features} & \textbf{Classes} & \textbf{Batchsize} & \textbf{Learning rate} & \textbf{\# partitions} & \textbf{\# iters.} \\ \hline
        Cora & 2,708 & 5,278 & 1,433 & 7 & 10 & 0.01 & 40 & 3 \\
        CiteSeer & 3,327 & 4,552 & 3,703 & 6 & 8 & 0.01 & 24 & 2 \\
        PubMed & 19,717 & 44,324 & 500 & 3 & 4 & 0.01 & 8 & 5 \\ 
        AmazonPhoto & 7,650 & 119,081 & 745 & 8 & 16 & 0.01 & 32 & 10 \\ 
        AmazonComputer & 13,752 & 245,861 & 767 & 10 & 16 & 0.01 & 32 & 8 \\ 
        CoauthorCS & 18,333 & 81,894 & 6,805 & 15 & 1 & 0.01 & 2 & 7 \\ 
        CoauthorPhysics & 34,493 & 247,962 & 8,415 & 5 & 2 & 0.01 & 4 & 9 \\ 
        WikiCS & 11,701 & 215,863 & 300 & 10 & 16 & 0.02 & 32 & 6 \\ \hline
    \end{tabular}
    \label{tab:datasets}
\end{table*}

\subsection{Results} \label{sec:results}
In this section, we present our experimental methodology and evaluate our WaveGAS method in comparison to GAS. We analyze WaveGAS, show its statistical significance and investigate the choice of number of iterations. 

For WaveGAS, we run a set of experiments with a different number of waveform iterations ranging from 1 to 11. A value of 1 indicates the standard GAS approach, more than 1 indicates our WaveGAS approach with multiple waveform iterations that do not track gradients followed by a single iteration with gradient tracking and backward pass. We choose the experiment and the corresponding number of iterations with the highest accuracy on the validation set and report the accuracy on the test set. These results are collected in Table \ref{tab:results} in the "WaveGAS (best \#iterations)" column. The table demonstrates that our WaveGAS approach consistently outperforms GAS (reported in the "GAS" column) across all datasets. Interestingly, the WaveGAS method even outperforms full-graph trained GNN (reported in the "Full" column) in some cases. In the column "WaveGAS (avg. iterations)" we also report the WaveGAS performance averaged for 5 different waveform iteration numbers taken from 2 to 6. 

In addition, we analyze the difference of the mean accuracy across all datasets compared to GAS. Choosing the best number of iterations (based on the validation set) leads to a delta of $+0.25$. Due to iteratively performing multiple forward passes, the network values are more accurate and the performance is improved. When investigating the average accuracy for a number of WaveGAS iterations 2 to 6, we can see that the difference of the mean accuracy across all datasets compared to GAS is $+0.17$. This further illustrates the effectiveness of WaveGAS, even with only a few WaveGAS iterations. 

Figure \ref{fig:wr} illustrates how the accuracy changes with the number of iterations taken up to 11. Interestingly, there is no clear linear increase in terms of accuracy with more iterations. Our hyperparameter configuration in Table \ref{tab:datasets} also supports this. 

It should be noted that the best number of WaveGAS iterations is chosen based on the best validation accuracy while we report the corresponding test accuracy. For this reason, it can occur that the average test accuracy of iterations 2 to 6 (reported in the "WaveGAS (avg. iterations)" column) is slightly higher than the accuracy with the best number of iterations as it is in the case of PubMed graph. 

\begin{table*}[t]
    \centering
    \caption{GAS vs. WaveGAS Performance (test accuracy, higher is better): The optimal number of WaveGAS iterations is determined based on the highest validation accuracy. For reference, the accuracy achieved with full-graph training is also reported.}
    \begin{tabular}{|c|c|c|c||c|}
    \hline
        \textbf{Dataset} & \textbf{GAS} & \textbf{WaveGAS} (best \#iterations) & \textbf{WaveGAS} (avg. iterations) & \textbf{Full} \\ \hline
        Cora & $81.54 \pm 2.43$ & $81.69 \pm 0.74$ & $\textbf{81.86} \pm 0.83$ & $81.62 \pm 0.78$ \\
        CiteSeer & $70.87 \pm 1.15$ & $\textbf{71.13} \pm 0.97$ & $70.73 \pm 1.09$ & $70.82 \pm 0.73$  \\
        PubMed & $78.89 \pm 0.69$ & $79.14 \pm 0.48$ & $\textbf{79.21} \pm 0.58$ & $79.18 \pm 0.54$ \\ 
        AmazonPhoto & $90.37 \pm 1.35$ & $\textbf{90.46} \pm 1.33$ & $90.16 \pm 1.38$ &  $90.08 \pm 1.63$ \\ 
        AmazonComputer & $ 80.42 \pm 1.76$ & $\textbf{81.34} \pm 1.98$ & $81.17 \pm 2.00$ & $80.79 \pm 2.01$ \\ 
        CoauthorCS & $90.66 \pm 0.67$ & $90.88 \pm 0.41$ & $90.79 \pm 0.51$ & $\textbf{90.94} \pm 0.76$ \\
        CoauthorPhysics & $92.57 \pm 1.07$ & $92.63 \pm 0.99$ & $92.75 \pm 0.95$ & $\textbf{93.02} \pm 0.69$ \\ 
        WikiCS & $78.78 \pm 0.55$ & $78.79 \pm 0.60$ & $78.76 \pm 0.61$ & $\textbf{78.96} \pm 0.51$ \\ \hline
        $\Delta$ Mean Accuracy & +0.00 & \textbf{+0.25} & +0.17 & \\ \hline
    \end{tabular}
    \label{tab:results}
\end{table*}

\begin{table*}[t]
    \centering
    \caption{The one-sided Wilcoxon test for the null hypothesis evaluates whether WaveGAS is less significant than GAS. A p-value below 0.05 indicates rejection of the null hypothesis.}
    \begin{tabular}{|c|c|c|c|c|c|c|c|c|c|c|} \hline
        \textbf{iterations} $I$ & 2 & 3 & 4 & 5 & 6 & 7 & 8 & 9 & 10 & 11  \\ \hline
        \textbf{p-value} & \textbf{0.0156} & 0.1250 & 0.2500 & 0.1250 & 0.0625 & 0.0625 & \textbf{0.0313} & \textbf{0.0156} & \textbf{0.0313} & \textbf{0.0313} \\ \hline
    \end{tabular}
    \label{tab:wilcoxon}
\end{table*}

In addition, we perform a one-sided Wilcoxon test \cite{wilcoxon1992individual} on all datasets and results to investigate the statistical significance of WaveGAS compared to GAS. A p-value of $<0.05$ is considered as a confirmation of the alternative hypothesis that WaveGAS is better than the baseline. Our results (Table \ref{tab:wilcoxon}) show that iterations 2, 8, 9, 10, and 11 yield significance. 

Generally, the training time increases linearly with more iterations, leading to longer training times of the WaveGAS method compared to GAS (Table \ref{tab:times}). As shown above, for most datasets up to 5 additional WaveGAS iterations results in a good performing accuracy. Consequently, there is only a slight increase in terms of training time. For instance, GAS is only $1.4\times$ faster than WaveGAS for CiteSeer and $1.8\times$ for CoauthorCS.

\begin{table}[t]
    \centering
    \caption{GAS vs WaveGAS training time (in s)}
    \begin{tabular}{|c|c|c|}
    \hline
        \textbf{Dataset} & \textbf{GAS} & \textbf{WaveGAS} \\ \hline 
        Cora & 95.0 & 280.0 \\ 
        CiteSeer & 103.6 & 145.6 \\ 
        PubMed & 132.7 & 752.0 \\ 
        AmazonPhoto & 112.2 & 690.7 \\ 
        AmazonComputer & 173.7 & 1180.0 \\ 
        CoauthorCS & 42.6 & 79.7 \\ 
        CoauthorPhysics & 1,398.3 & 11,017.0 \\ 
        WikiCS & 145.0 & 568.0 \\ \hline 
    \end{tabular}
    \label{tab:times}
\end{table}

\begin{figure}[t]
    \includegraphics[width=\columnwidth]{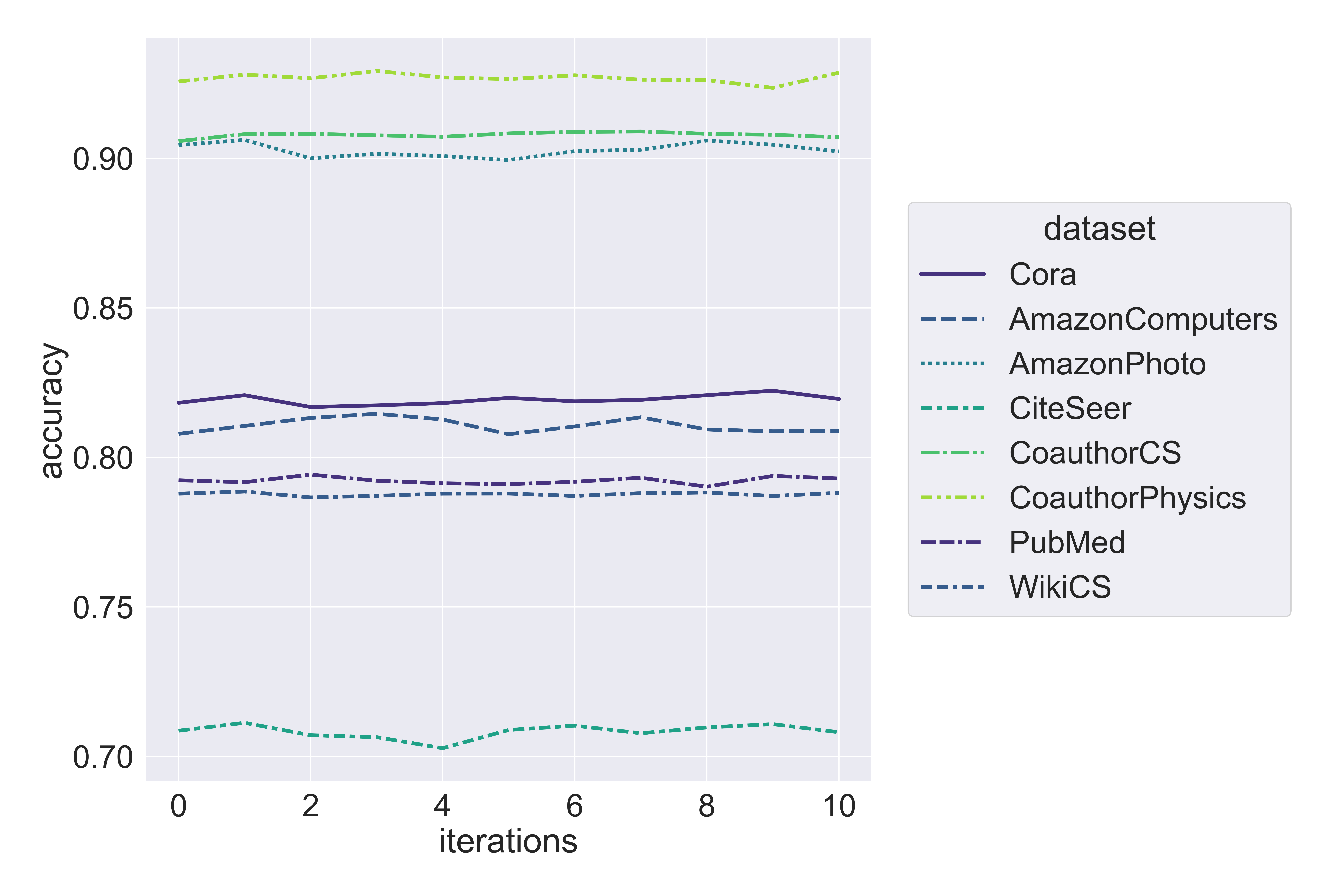}
    \caption{\label{fig:wr} WaveGAS accuracy versus the number of Waveform Relaxation iterations for all datasets}
\end{figure}

\subsection{Limitations}
Our experiments have shown that our WaveGAS method works very well and is able to improve significantly upon the GAS method. However, it should be noted that there are a few limitations. 
The training time of WaveGAS increases with more WaveGAS iterations compared to GAS only. However, among the most significant numbers of iterations is only 1 additional iteration meaning the training time only slightly rises in most cases. This is also supported when investigating the average test accuracy of 1 to 5 additional WaveGAS iterations. Here, we also experience an improvement compared to GAS while still training in an efficient manner. During inference, WaveGAS needs multiple forward passes as it does during training.

\section{Unsuccessful attempts}
When working on improving GAS and developing WaveGAS, we encountered some approaches that were not as promising as expected.

Our proposed gradient-correcting method GradAS is able to slightly improve the accuracy compared to GAS for some datasets (e.g., Cora $82.05 \pm 0.91$). However, the results vary and not clearly outperform GAS. In addition, the training time and memory consumption is quite high due to the computation and tracking of the historical gradients with the jacobian matrix. For Cora, the training time is 3 times higher compared to GAS. While GAS needs 1.5 GB of GPU memory, GradAS consumes around 11GB of GPU memory. Its resource-intensive nature reduces its practical use. 

\section{Related Work}
GNNAutoScale (GAS) \cite{fey2021gnnautoscale} partitions the graph and performs training on each of the partitions. To handle cross-partition links, historical embeddings are stored. Instead of loading and computing the full computation graph, GAS stays within the current partition and loads the historical embeddings of out-of-partition nodes to avoid information loss. This reduces training time and memory consumption. However, historical information can be stale since historical embeddings of one-hop nodes are not updated during the GNN training on the target partition, and approximation errors can occur. We distinguish ourselves by focusing on improving the historical embeddings in the GAS model to mitigate information loss. 

Another approach that partitions the graph is Cluster-GCN \cite{chiang2019cluster}. Within the partitions, sampling is performed and the GNN is trained. Due to the limitation of message passing to the current minibatch, potentially useful information outside the current partition is lost. Similar to ClusterGCN, GraphSAINT \cite{zeng2020graphsaint} also first clusters the graph to avoid neighborhood explosion and then forms minibatches inside the partitions. Information across partitions is ignored, leading to a potential decrease of the models performance. 

DistDGLv2 \cite{zheng2022distributed} employs an efficient GNN training method based on a synchronous training approach with an asynchronous minibatch generation pipeline. Resource usage is optimized, but the loss function does not consider node embeddings across different partitions. 

Besides a mini-batch training method which supports a high degree of parallelism, ByteGNN \cite{zheng2022bytegnn} proposes a graph partitioning method adapted to GNN workloads. Depending on the target nodes, the partitioning is adapted to include the k-hop neighborhood of a node. This minimizes data movement and communication. Links across partitions are not considered in the loss function. Therefore, training performance could be affected. 

Betty \cite{yang2023betty} introduces redundancy-embedded graph (REG) partitioning and memory-aware partitioning. Redundancy is mitigated, and load balance across partitions is improved. Further, the authors move from mini-batch training to micro-batch training. For each micro-batch, partial gradients are calculated which are accumulated to update the weights of the full model. In this way, memory is reduced while bypassing the loss of information when using partitioning and mini-batch training, but information is still lost due to links across micro-batches. 

\citet{li2021training} focus on training deep GNNs. The authors adopt grouped reversible GNNs which partitions tensors of initial features in groups and employs reversible residual connections. This network only stores final outputs and thus saves memory. While their approach provides fixed memory requirements irrespective of the number of layers, our approach provides fixed memory requirements irrespective of the layer width. Therefore, our approach is complementary and can be used in combination.

\citet{xue2024haste} introduce the REST method to address feature staleness by highlighting the importance of refreshing stored embeddings more frequently to align with model updates. Their approach involves performing a forward pass with batches formed according to a certain split, followed by a forward and backward pass with another split. 
In contrast, we propose a novel refinement process that performs multiple forward passes on all batches with an unchanged given split, ensuring that all historical embeddings are updated relative to a consistent state of the network parameters. This is followed by a backward pass designed specifically to enhance the robustness of historical embedding approximations. Additionally, we are the first to establish a connection between this approach and the well-known Waveform Relaxation method for solving ODEs. Furthermore, we provide an in-depth analysis of the gradient discrepancy issue in the GAS method, offering new theoretical insights that extend the understanding of feature staleness mitigation.

\section{Conclusions and Future Work}
Our work addresses the issue of stale embeddings when using the GAS method. We propose two enhancements: first, WaveGAS, which, inspired by the waveform relaxation, refines the approximation of the embeddings and gradients by performing multiple forward passes with GAS before the backward pass. Second, a method tracking gradients during training resulting in more accurate historical gradients. Our experimental results show that especially WaveGAS enhances GAS and achieves better accuracy. Compared to GAS, the mean accuracy rises by $+0.25$ across all datasets when choosing the best number of iterations according to the validation set. 

Future work could investigate our WaveGAS method with more GNN architectures, such as GAT \cite{velivckovic2018graph} or GIN \cite{xu2018powerful}. Further, it could be explored if by using less additional iterations in later training epochs, the accuracy of WaveGAS can be maintained while decreasing the training time. 

\section*{Impact Statement}
Training GNNs on large graphs has significant practical applications across various domains. Building on GAS, which introduced a method for scaling GNN training using historical embeddings, we conduct a comprehensive analysis of its theoretical limitations. To address these, we adapt well-established numerical methods for solving ODEs, and propose enhancements to improve the GAS framework. Our empirical results highlight the effectiveness of the proposed methods, and we also document unsuccessful attempts, shedding light on unresolved challenges. These insights pave the way for future research to further address the complexities of scaling GNN training.

\bibliography{bib}
\bibliographystyle{icml2025}

\end{document}